\documentclass[10pt,conference]{ieeeconf} 

\IEEEoverridecommandlockouts 
\usepackage[utf8]{inputenc}
\usepackage{graphicx}
\usepackage{amsfonts}
\usepackage{amsmath}
\usepackage{amssymb}
\usepackage{tabularx}

\usepackage{algorithm}
\usepackage{algpseudocode}
\usepackage{xcolor}
\usepackage{placeins}

\def\vu{\boldsymbol{u}} 
\def\vx{\boldsymbol{x}} 
\def\vv{\boldsymbol{v}} 
\def\vs{\boldsymbol{s}} 

\def\setX{\mathcal{X}}

\def\setV{\mathcal{V}}
\def\setE{\mathcal{E}}
\def\setXdense{\setX_{dense}} 
\newtheorem{problem}{Problem}
\newtheorem{theorem}{Theorem}
\def\gve{\mathcal{G(\setV,\setE)}}

\title{\LARGE \bf Dispersion-Minimizing Motion Primitives \\for Search-Based Motion Planning
}
\author{Laura Jarin-Lipschitz, James Paulos, Raymond Bjorkman, and Vijay Kumar
\thanks{We gratefully acknowledge the support from ARL Grant DCIST CRA W911NF-17-2-0181, NSF Grant CNS-1521617, ARO Grant W911NF-13-1-0350, ONR Grants N00014-20-1-2822 and ONR grant N00014-20-S-B001, Qualcomm Research, and C-BRIC, a Semiconductor Research Corporation Joint University Microelectronics Program  program cosponsored by DARPA. The first author acknowledges support from the NSF Graduate Research Fellowship Program. We would also like to thank Mathew Halm for his feedback.}
\thanks{The authors are with the GRASP Laboratory, University of Pennsylvania, PA, 19104, USA {\tt\footnotesize \{laurajar, jpaulos, raybjork,  kumar\}@seas.upenn.edu.}}%
}

\begin{document}

\maketitle

\begin{abstract}

Search-based planning with motion primitives is a powerful motion planning technique that can provide dynamic feasibility, optimality, and real-time computation times on size, weight, and power-constrained platforms in unstructured environments.  However, optimal design of the motion planning graph, while crucial to the performance of the planner, has not been a main focus of prior work. This paper proposes to address this by introducing a method of choosing vertices and edges in a motion primitive graph that is grounded in sampling theory and leads to theoretical guarantees on planner completeness.  By minimizing dispersion of the graph vertices in the metric space induced by trajectory cost, we optimally cover the space of feasible trajectories with our motion primitive graph. In comparison with baseline motion primitives defined by uniform input space sampling, our motion primitive graphs have lower dispersion, find a plan with fewer iterations of the graph search, and have only one parameter to tune.

\end{abstract}

\section{Introduction}
Search-based planning with motion primitives is a motion planning technique that is especially performant on small, size, weight, and power-constrained platforms in unknown and unstructured environments. It has been widely adapted for planning for a variety of systems, including autonomous vehicles \cite{Pivtoraiko2009} \cite{Likhachev} \cite{Sun2020}, robotic arms \cite{Likhachev2010}, Micro Aerial Vehicles (MAVs) \cite{Pivtoraiko2013} \cite{Liu2017} \cite{Dharmadhikari2020}, and multi-robot systems \cite{Thakur2013}. However, prior work does not provide a focus on designing the graph in order to optimize planner performance. In this work, we focus on optimizing the design of the motion planning graph for computation time and planner completeness.


\begin{figure}
	\centering
	\includegraphics[width=0.99\columnwidth]{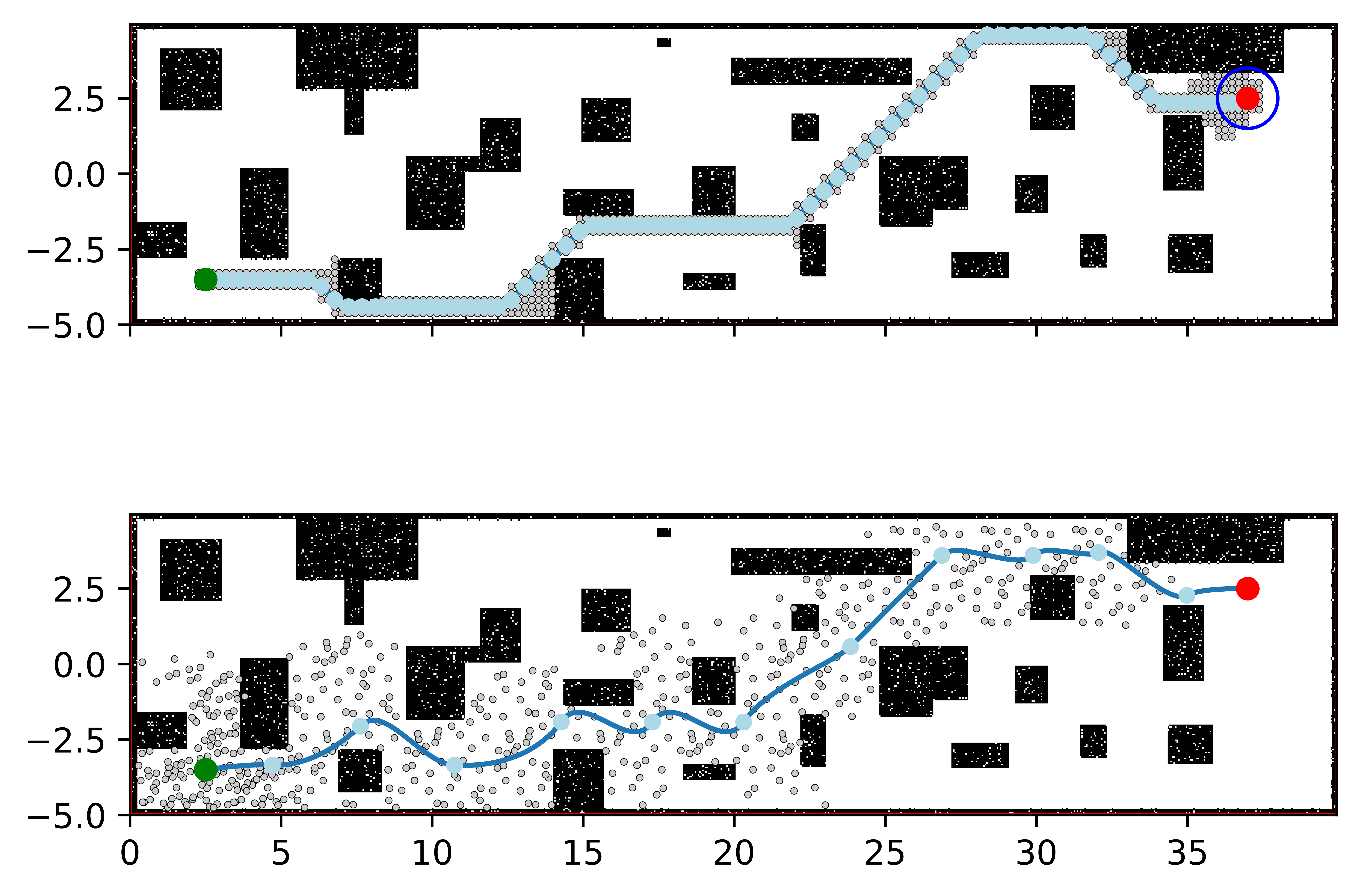}
	\caption{This figure shows the graph search planning results compared between uniform input sampling (top) and our method (bottom). They share the same dynamical system, the quadrotor double integrator, and the same state and input constraints. In order to compare them, we show an example with approximately equally optimal plans (ours is 3.5\% more optimal), but our method requires 1.74x fewer collision checks (iterations through the graph search).}
	\label{fig:planning2}
	\vspace{-5mm}
\end{figure}

In order to perform search-based motion planning with motion primitives, any approach must address the question of which states (vertices) and trajectories (edges) compose the planning graph. However, the more vertices and edges in the planning graph, the more the computational cost of our graph search will grow. Prior work relies on motion primitives designed manually \cite{Likhachev2010}, derived from coarse, regular state discretization \cite{Pivtoraiko2013} \cite{Ljungqvist2017}, or obtained by sampling constant inputs \cite{Liu2017}. Uniform/regular sampling in the state space is computationally inefficient due to the `curse of dimensionality:' sampling in possibly high-dimensional state spaces scales exponentially with the number of dimensions, which especially inhibits usage for realistic robotic systems that must include velocities in their state. Alternatively, sampling in the generally lower dimensional input space as in \cite{Liu2017} uses computation of forward trajectory simulation which can be fast online. However, given that the planning query is posed in the state space, the selection of a satisfactory set of primitives from the input space is not obvious for nontrivial system dynamics. 

To optimally place our vertices, we look to recent work \cite{Palmieri2019} that presents an optimization algorithm to generate approximately minimum dispersion samples in the metric space induced by trajectory cost for use in planning.  \emph{Dispersion} \cite{nied} \cite{LaValle2006} is a metric from sampling theory that quantifies the largest region of a metric space that does not contain a sample (vertex). However, their method is only applicable to `driftless' systems, meaning that it cannot apply to real-world dynamical systems with momentum. Additionally, their planner is confined to fixed size configuration spaces, making it again unsuitable for robots in the wild which require long range maneuvers. This leads us back to the advantage of search-based planning with motion primitives. By using dispersion optimization to select a \emph{tiled} graph of motion primitives, offline, we are able to create a fast, online graph search-based planner for real robotic systems. 


The main contributions of the paper are as follows: 
\begin{enumerate}
    \item A computationally feasible offline algorithm to generate minimum dispersion vertices for systems for which we can compute an optimal steer function, while respecting state and input constraints and robot dynamics. 
    \item A method to embed the minimum dispersion vertices into a search-based motion planning graph. This planner is proven to have a theoretical completeness guarantee in Theorem 1.
    \item A method for tiling the optimally sparse motion primitive graph (computed offline once per system), many times online to plan in arbitrary size configuration spaces.
    \item A computational comparison of our minimum dispersion graph planner to a baseline uniform input sampling search-based planner.

\end{enumerate}



\begin{figure}
	\centering
	\includegraphics[width=1\columnwidth]{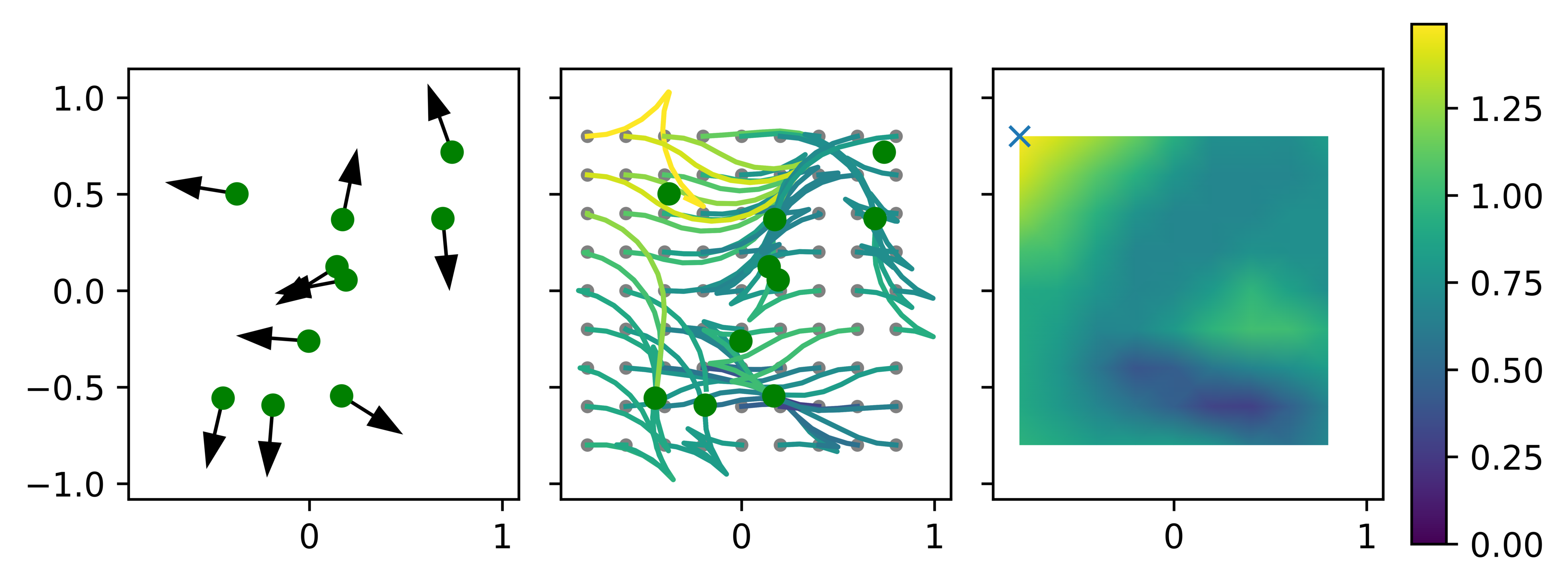}
	\caption{(left) A random set of states $\setV$ in $(x,y,\theta)$ for the Reeds-Shepp car. (center) Minimum cost trajectories from $\setV$ to grey sampled states $\setXdense$ with coordinates $(x,y,0)$ are shown with color indicating the trajectory cost. (right) The resulting cost surface over sample states $\setXdense$. The sampled state with the maximum cost is marked by the blue X and lower bounds the dispersion of $\setV$.}
	\label{fig:heatmap}
	\vspace{-5mm}

\end{figure}

\section{Problem Definition}

\subsection{Dynamical System}

We are motivated by solving motion planning problems for a dynamical system subject to configuration, state, and/or input constraints.
The dynamical system is governed by $\dot{\vx} = f(\vx, \vu)$ with state $\vx(t) \in \mathcal{X} \subset \mathbb{R}^n$ and inputs $\vu(t) \in \mathcal{U} \subset \mathbb{R}^m$.
In addition, we have a strictly positive running cost \(L\) which defines the cost functional $J(\vu(t)) = \int_{t_0}^{t_f} L(\vx(t), \vu(t), t)$ giving the cost of any admissible trajectory from time $t_0$ to time $t_f$.

We will assume that we have access to a steering function $\vu^{*}(\vx_1,\vx_2)$ which provides the minimum cost  trajectory from any state $\vx_1$ to any state $\vx_2$ in free space, ignoring any configuration constraints. For example, this could be the output from a standard trajectory optimization problem.
The steering function additionally provides the optimal cost itself, which with a slight abuse of notation we write as $J(\vx_1,\vx_2)$.
Note that in general the cost $J(\vx_1, \vx_2)$ may not be symmetric, in which case $\exists \vx_1,\vx_2 : J(\vx_1,\vx_2) \neq J(\vx_2,\vx_1)$.
$J(\vx_1,\vx_2)$ satisfies all other axioms of a metric and so is a \emph{quasimetric}. 

\subsection{Motion Primitive Graph} 
Let $\mathcal{G(\setV,\setE)}$ be a graph that is composed of a finite set of vertices $\setV \subset \setX$ and edges $ (\vv_1,\vv_2) \in \setE \subset \setX \times \setX$ Each edge represents the existence of a plan in free space to get from $\vv_1$ to $\vv_2$ with cost $J(\vv_1,\vv_2)$. A wide class of search algorithms exist for optimal planning within graphs. In this instance, since we have access to an optimal steer function, all of the edges are dynamically feasible to traverse.

\subsection{Dispersion} \label{sec:dispersion}
The distribution of finite reachable states obtained by our motion primitives may be characterized by its \emph{dispersion}.
We use the definition of the dispersion as in \cite{nied} \cite{LaValle2006}, defined over $\setX_{free}$ as follows:
\begin{align} \label{eq:disp}
    d(\setV) = \sup\limits_{\vx \in \setX_{free}} \left[ \min_{\vv \in \setV} J(\vx, \vv)\right]
\end{align}
The dispersion is the radius of the largest empty norm ball in the space induced by a metric $J$ that does not intersect a sample set $\setV$.
However, we adapt this definition to apply to our general \emph{nonsymmetric} system by ensuring that both the forward and backward cost of the quasimetric obey the dispersion property, the consequences of which are explored in more detail in Section \ref{sec:proof}.

Using this definition we can now define the non-symmetric dispersion as follows:
\begin{align}
    d(\setV) = \sup\limits_{\vx \in \setX_{free}} \left[ \min_{\vv \in \setV} [\max(J(\vx, \vv),J(\vv,\vx))]\right]
\end{align}

This is the same as equation \ref{eq:disp}, except with a new metric $\bar{J}(\vx, \vv)= \max(J(\vx, \vv),J(\vv,\vx))$.

Since the space we are dealing with is connected by feasible dynamically trajectories, the region around a state $\vx$ with a cost no more than $j$ (the ball) is the cost-limited reachable set of the system. The cost-limited reachable set is defined as follows:
\begin{align}
R(\vx,j) = \{ \vx_i \in \setX| J(\vx,\vx_i) \leq j\}
\end{align}

An equivalent and more intuitive way to describe dispersion is therefore to say that if the dispersion is no more than $d$, then:
\begin{align}
\forall \vx \in \setX, \exists \vv \in \setV, \vx \in R(\vv,d) \land \vv \in R(\vx,d)
\end{align}
For every point in the state space, there exists a trajectory to and from the sample set $\setV$ that has cost less than or equal to $d$.


Figure \ref{fig:heatmap} illustrates the concept of dispersion using the control cost for an example system. We consider Reeds-Shepp car dynamics with state $\vx \in SE(2)$ having coordinates $(x,y,\theta)$ and for which there is a well known optimal, symmetric steering function providing minimum distance trajectories from any state $\vx_1$ to $\vx_2$. The left panel shows a random set of states $\setV$ for which we will assess the dispersion. We compute the minimum cost trajectories from $\setV$ to reach points in a dense sampling of the state space $\setXdense$ and vice versa. The center panel depicts $\setXdense$ in grey along with the linking trajectories colored to indicate their cost. Note that to simplify the illustration we have selected a system with symmetric costs and are only illustrating a dense sampling over $(x,y,0)$. The costs of these trajectories define the cost surface shown in the right panel. The state marked with an X has the highest cost, and that cost lower bounds the dispersion of $\setV$. Further details on numerically approximating the dispersion are discussed in Section \ref{sec:method}.

\subsection{Dispersion-Optimal Motion Primitives}

The design of the motion primitive set has a large effect on the completeness, planning time, and motion plan quality obtained through search based planning.
Instead of manually designing motion primitives or arbitrarily choosing constant input segments we propose to select a motion primitive set that minimizes the dispersion of the reachable states.

\begin{problem} Dispersion-Optimal Motion Primitives \\
Find a set of motion primitives $\gve$ with minimum $|\setV|$ such that the dispersion of the discrete reachable states $\mathcal{X}_R \in \mathcal{X}$ is less than $d$.
\end{problem}


\section{Minimum dispersion motion primitive graphs} \label{sec:method}

In this section, we detail our method to \emph{approximately} solve Problem 1 by constructing a minimum dispersion graph $\mathcal{G}(\setV,\setE)$. At a high level, we use Algorithm \ref{alg:disp} to numerically approximate dispersion in the metric space defined by $\bar{J}$, and iteratively choose samples to add to $\setV$ that reduce this dispersion. Though this is computationally expensive and involves computing large numbers of optimal steer function, this all happens \emph{offline} with respect to planning. Algorithm \ref{alg:disp} outputs the motion primitive graph, which is stored and reconstructed for fast \emph{online} use with graph search planning.




\begin{figure}
	\centering
	\includegraphics[width=.9\columnwidth]{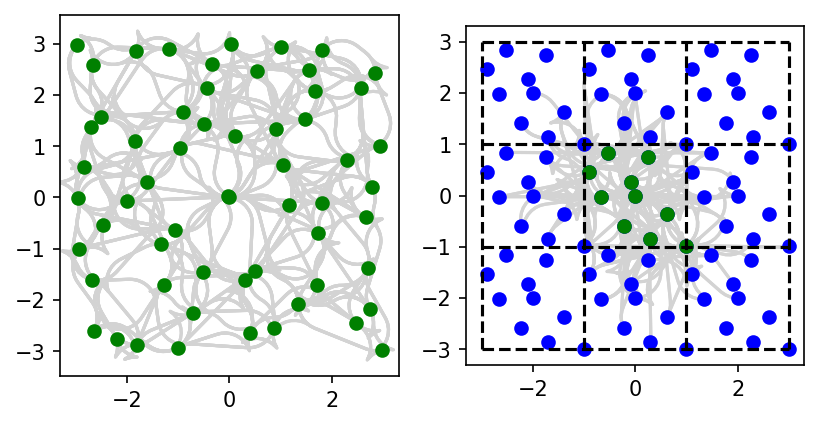}
	\includegraphics[width=0.5\columnwidth]{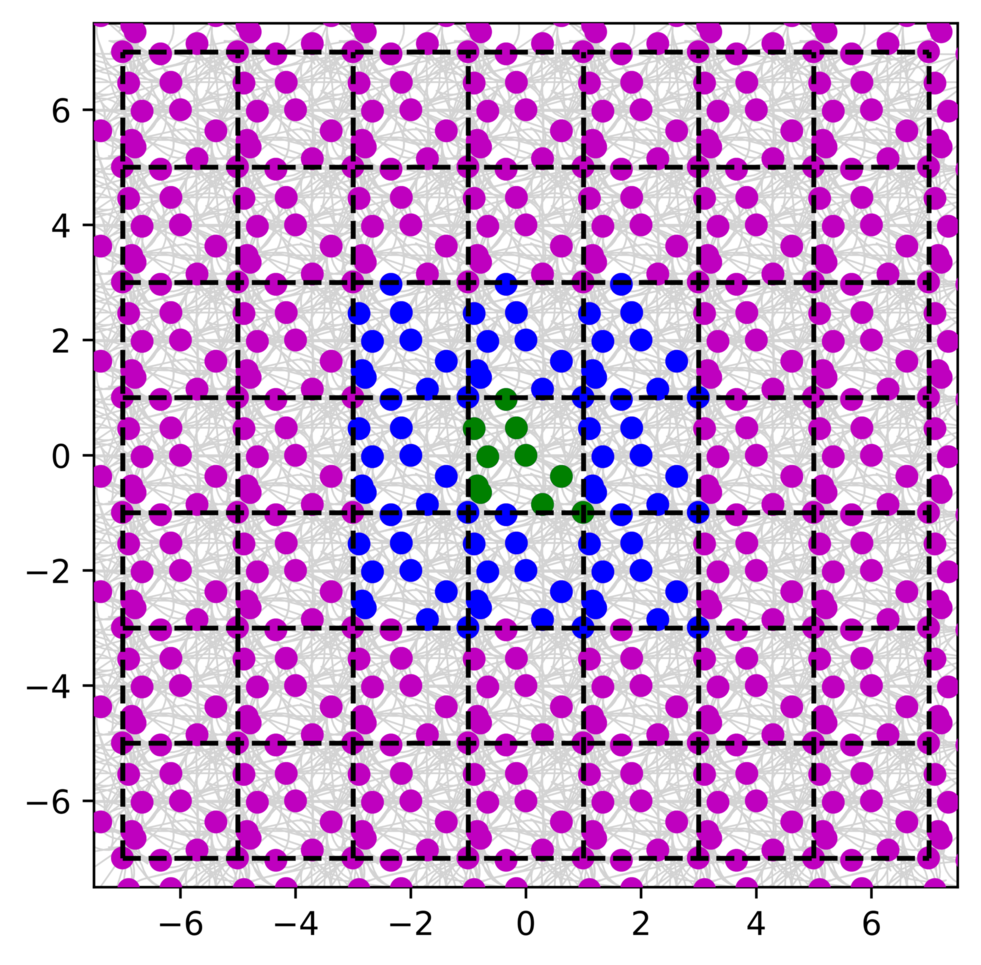}

	\caption{In all three plots, the green dots are the vertices $\setV$ output by Algorithm \ref{alg:disp}, projected onto the configuration space. The dynamically feasible trajectories representing edges $e \in \setE$ that have cost $< 2d$ (refer to section \ref{sec:proof}), are represented in light grey. The system used in this example is the Reeds-Shepp car (state x,y, and $\theta$), and the dispersion is 1.05. (top left) The samples are chosen according to Algorithm \ref{alg:disp}, but without line 4, the tiling step, s.t. $\setX_{tile} = \setV$. 
	(top right)  The samples $\setV$ are chosen according to Algorithm  \ref{alg:disp}. The blue dots are the `tiled' vertices, $\setX_{tile}$, and the dotted lines represent the boundaries of each translated copy of $\setV$.
	(bottom) The same graph is shown with even more tiled copies of the graph, as will be used at planning time. The pink samples are not considered during Algorithm \ref{alg:disp}.
	}
	\label{fig:tiling}
	\vspace{-5mm}

\end{figure}

\begin{algorithm}

    \caption{Minimum Dispersion Vertices}
    $\setXdense$, a low-discrepancy sampling of the \textit{state space} \\
    $\setV$, the output minimum dispersion vertices \\
    $\setX_{tile}$, $\setV$ translated in all spatial dimensions \\
    $\mathbf{J_{min}}[\vx_s]$, min. cost from $\vx_s \in \setXdense$ to any $\vx_t \in \setX_{tile}$ \\
    $D$, target dispersion 
    \begin{algorithmic}[1] 
        \Procedure{MinimumDispersionVertices}{$D$}
        \State $\vv \leftarrow zeroState() $,  $\setV \leftarrow [\vv ] $, $d \leftarrow \infty$
        \While{$d > D$}
        \State \textcolor{black}{$\setX_{tile} \leftarrow tilePoints(\setV)$}

        \ForAll{$\vx_s \in \setXdense$}
        \ForAll{$\vx_t \in \setX_{tile}$}
        \State {$j \leftarrow \max ({J(\vx_s,\vx_t)},{J(\vx_t,\vx_s)})$}
        \State $\mathbf{J_{min}}[\vx_s] \leftarrow \min(\mathbf{J_{min}}[\vx_s],j)$
        \EndFor
        \EndFor
        \State $\vv \leftarrow \arg\max \mathbf{J_{min}}[\vx_s]$, $append(\setV, \vv) $\;
        \State $d \leftarrow \max{\mathbf{J_{min}}[\vx_s]}$
        \EndWhile
        \State \Return $\setV$

        \EndProcedure
    \end{algorithmic}
    \label{alg:disp}
\end{algorithm}

\subsection{Computing Dispersion Numerically}
Given a set of samples and a metric, in general it is not possible to analytically compute dispersion \cite{nied}. Choosing a set of samples to minimize the dispersion, given the cost function is equally analytically intractable. These roadblocks lead to a numerical approach in order to compute dispersion and to generate a minimum dispersion sample set. 

To approximate dispersion numerically, we take a dense set of samples from the state space $\setXdense$, and compute trajectories from each of these points to all points in the sample set, as well as in the reverse direction. This allows us to compute the metric $\bar{J}$. Following from the continuous definition of dispersion described in Section \ref{sec:dispersion}, the discrete approximation of dispersion is that for every state in $\setXdense$, there exists a trajectory to and from $\setV$ that has cost less than the dispersion.

\subsection{Minimum Dispersion Vertex Selection}
The underlying structure of the dispersion optimization algorithm in \cite{Palmieri2019} is adapted in this work into Algorithm \ref{alg:disp}, but with key differences that allow us to plan for a wide class of dynamical systems, instead of only for driftless dynamical systems, and over arbitrary distances instead of a fixed size workspace. As in the prior work, the vertices $\setV$ are selected from a more numerous set $\setXdense$ over the Euclidean state space. The dispersion in $\bar{J}$ is then computed, and the dense sample at the highest dispersion state is greedily added to $\setV$.  Due to the greediness of Algorithm \ref{alg:disp}, it does not provide a true global optimum. However, as we can see in Figure \ref{fig:heatmap2}, the algorithm does significantly reduce dispersion compared to our baseline planner's sampling; which helps our planner completeness as we describe later. 
The rest of this section details how Algorithm \ref{alg:disp}  modifies \cite{Palmieri2019} to address non-symmetric costs and infinite workspaces. 


\subsubsection{New metric for non-symmetric systems}
Owing to the way we defined dispersion with $\bar{J}$ in Section \ref{sec:dispersion}, we must compute trajectories both from the vertices to the dense sampling, and vice versa. The cost between the states is the maximum of the costs of these two trajectories, as we defined in section \ref{sec:dispersion}. 

\subsubsection{Choosing $\setXdense$}
$\setXdense$ is chosen to be a low-discrepancy Sobol sequence \cite{nied}. In the Euclidean metric spaces, we already have a way to choose low dispersion points. Deterministic low discrepancy sequences imply low dispersion \cite{LaValle2006}. By using a sampling that is lower dispersion than uniform sampling in the Euclidean metric, we also reduce $\bar{J}$ dispersion in $\setX_{dense}$, of which $\setV$ is a subset.

\subsubsection{Tiling - Creating a Graph in an Unbounded Configuration Space}
Due to our system property that the system dynamics are \emph{not} a function of position, a minimum dispersion graph in a \emph{finite} configuration space can be reused to plan in an \emph{infinite} configuration space. In order to accomplish this while maintaining our dispersion value, when we compute dispersion we consider `outgoing' trajectories from the graph that map back to vertices on copies of the initial graph that have been translated in space. 

The `tiled vertices set' $\setX_{tile}$ is constructed by translating the entire set $\setV$ every positive and negative combination of one bounding box length in the spatial dimensions (in Algorithm 1 this is called $tilePoints$. This is visualized in Figure \ref{fig:tiling}, where the green and blue vertices combined comprise $\setX_{tile}$.

When we compute dispersion in the midst of Algorithm \ref{alg:disp}, we consider these tiled vertices the same as members of $\setV$. This preserves our approximation of dispersion inside and between each tile, since we compute dispersion with the `next tile' over in the calculation. 

An illustration of the utility of this decision is shown in Figure \ref{fig:tiling}. Both of the plots in the example shown achieve the same dispersion, but the tiled version requires maintaining a dictionary of 10 vertices and 224 edges, while the non-tiled version requires 60 vertices and 668 edges. However, this does come at some cost, since the tiled version now has about 22 edges per vertex (the branching factor), compared to 11 for the non-tiled one, which can negatively impact planning time. However, the most important outcome from this decision is that we are no longer confined to the bounded configuration space required by the dispersion algorithm. Our graph is transformed to be  \emph{implicit} instead of \emph{explicit} one, usable over infinite configuration spaces.


\begin{figure}
	\centering
	\includegraphics[width=0.55\columnwidth]{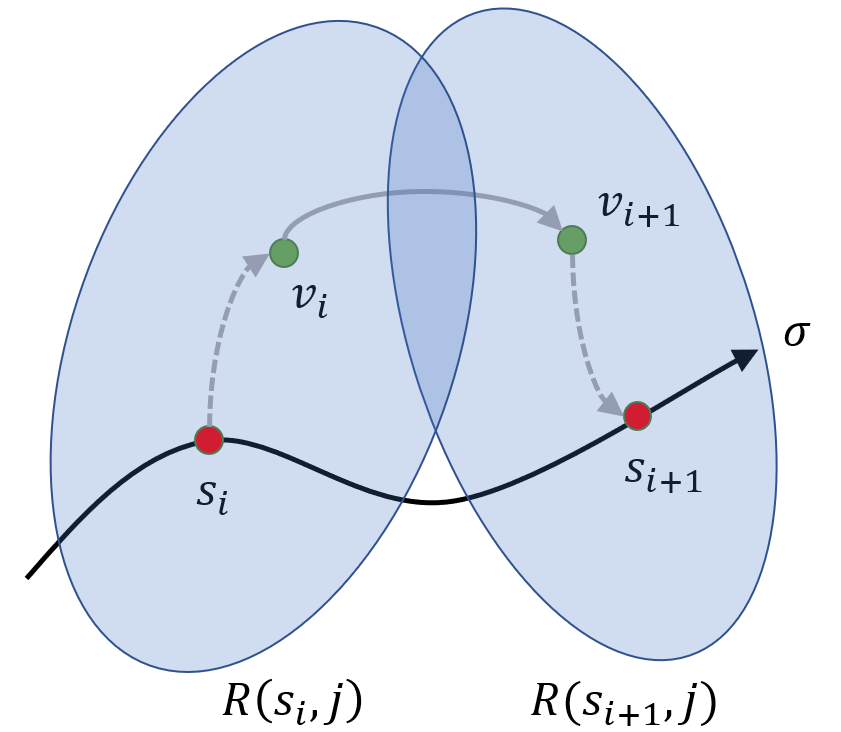}
	\caption{Reference for section \ref{sec:proof}, the proof of completeness. If there exist vertices `nearby' to the (unknown but assumed to exist) optimal path $\sigma$, they are connected by an edge in our graph, and that space is inside $\setX_{free}$ our planner can achieve completeness.}
	\label{fig:proof}
	\vspace{-5mm}

\end{figure}

\subsection{Minimum Dispersion Graph Search} \label{sec:graph_search}
Following Algorithm \ref{alg:disp}, we have generated a set of vertices $\setV$ that we store. As we will justify in the following section, we construct our final motion primitive graph $\gve$ by adding edges to $\setE$ for \emph{all} pairs of vertices that have cost in the metric less than twice the dispersion. With this graph, we want to solve motion planning problems to output feasible trajectories $\sigma(t)$that go from a $\vx_{start}$  to $\vx_{goal}$ over time $t_0$ to $t_f$.

Given a planning query $(\vx_{start},\vx_{goal})$, the only modification we need to using classical graph search is accessibility and departability to our $\gve$.
\subsubsection{Accessibility} Given the start state $\vx_{start}$, we normalize the position to zero, and use the steer function to connect to all $\setV$, and add these neighbors to the open list.
\subsubsection{Departability} Since we have access to the steer function, we can check if each open node can be connected directly to the goal state during the graph search. To avoid online computation of the steer function, we can alternatively define a terminating state set enclosed by the backwards reachable set of $\vx_{goal}$ up to cost $d$.



\subsection{Completeness} \label{sec:proof}
We have now accomplished our aim of designing an algorithm to output minimum dispersion vertices $\setV$, but are still left with the selection of edges $\setE$. We are motivated by a desire for planner completeness, to find a plan if one exists, and terminate our graph search if one does not. Due to discretization, we will never be complete in every case (for that we would need to be able to visit every $\vx \in \setX$), so we aim to quantify under what conditions we can guarantee that our planner is complete. The immediate consequence of this analysis will be that we will connect any two vertices in $\setV$ if their edge cost is less than twice the dispersion, and that we are not guaranteed to find a plan if the true optimal plan is too close to an obstacle.

In order to analyze completeness, we must now begin to consider the obstacles contained in $\setX$. To relate our system to these obstacles, we use a new quantity as defined in \cite{Palmieri2019}, the clearance $\delta$ of a trajectory $\sigma(t)$.
\begin{align}
 \delta(\sigma(t)) = sup \{ j \mid R(\vx,j) \subseteq \setX_{free} \forall \vx \in \sigma(t) \}
\end{align}

We now can investigate the relationship between clearance, dispersion, and completeness. We prove a theorem that says that with a optimal path clearance of twice the dispersion, and a graph with all vertices with cost less than twice the dispersion connected, our graph search is complete. In other words, clearance tells us that a tube around the optimal plan is free, dispersion of the vertex set tells us that our motion planning graph has samples in that tube, and defining connection rules for the vertices (edges) allow us to ensure that we can combine these vertices into a trajectory. Figure $\ref{fig:proof}$ illustrates the the mechanism of the proof. 
\begin{theorem}
Suppose $\vx_{start}$ and $\vx_{goal}$ can be connected by a motion plan $\sigma$ with clearance $\delta > 0$. If we have a graph $\gve$ which has dispersion $d(\setV) < \delta/2$ and all vertices $\vv_i,\vv_j$ with $J(\vv_i,\vv_j) < \delta$ are connected by an edge, a graph search through $\gve$ with accessibility and departability as defined in Section \ref{sec:graph_search}, a graph search through $\gve$ will find a path from $x_{start}$ to $x_{goal}$ through free space.
\end{theorem}

\begin{proof}
Let clearance $\delta > 0 $ and let $\gve$ be a motion primitive graph with dispersion $d(\setV) \leq \delta/2$, and let all pairs of vertices with cost less than $2d$ be connected. Let $\vx_{start},\vx_{goal}$ be connected with a motion plan $\sigma(t)$ with clearance $\delta(\sigma) \geq \delta$. Let $0 < \epsilon < \delta/2$ such that there are no pair of vertices in the motion planning graph with $\delta < J(\vv_i,\vv_j) \leq \delta + \epsilon$. We define a set of $n$ discrete states on this motion plan $s_i = \sigma(t_i)$ such that the first state $\vs_1$ is the start, the last state $\vs_n$ is the goal, and the intermediate states are separated by a sufficiently small cost, such that $J(\vs_i,\vs_{i+1}) < \epsilon$. 

First, we will check that our motion planning graph has vertices nearby to the optimal trajectory $\sigma$ and that they are connected. Since dispersion $d(\setV) \leq \delta/2$,  $\forall s_i, \exists v_i ~s.t.~ \bar{J}(v_i, s_i) \leq \delta/2$. By the triangle inequality (refer to Figure \ref{fig:proof}), 
\begin{align*}
J(\vv_i,\vv_{i+1}) & \leq J(\vv_i,\vs_i) + J(\vs_i,\vs_{i+1}) + J(\vs_{i+1},\vv_{i+1}) \\
 & \leq \bar{J}(\vv_i,\vs_i) + J(\vs_i,\vs_{i+1}) + \bar{J}(\vs_{i+1},\vv_{i+1}) \\
 & \leq \delta/2 + \epsilon + \delta/2 = \delta + \epsilon 
\end{align*}
Therefore, from our construction of $\epsilon$,  $J(\vv_i,\vv_{i+1}) < \delta$, so $\vv_i$ and $\vv_{i+1}$ must also be connected by an edge to the graph. 

Next, we check that these edges are also contained in free space. We know that the region around each state $\vs_i$ is free because of our definition of the clearance: the obstacle free regions $R(\vs_i,\delta)$ and $R(\vs_{i+1},\delta)$ are contained in free space. If we bisect in cost the trajectory from $\vv_i$ to $\vv_{i+1}$, each half has cost $\delta/2$. Therefore, $J(\vs_i,\vv_{i+.5})$ and $J(\vv_{i+.5}, \vs_{i+1})$ are both $\leq \delta$. Therefore, the path from $v_i$ to $v_{i+1}$ is also fully contained in free space. Similarly, from the start $\vs_1$ to the first vertex in the motion planning graph $\vv_1$ and from the goal $\vs_n$ to the last vertex in our planned sequence $\vv_n$ are also contained in free space due being contained in $R(\vs_1,\delta)$ and $R(\vs_n,\delta)$, respectively, giving us accesibility and departibility to the graph.

Therefore, a motion plan following the sequence $\vx_{start}, \vv_1, \dots, \vv_n, \vx_{goal}$ will be found by the graph search planner, and is in free space.

\end{proof}

We can also note that in the limit as dispersion goes to zero, our path will be the optimal path, since $v_i$ will equal $s_i$.

\begin{figure}
	\centering
	\includegraphics[width=0.99\columnwidth]{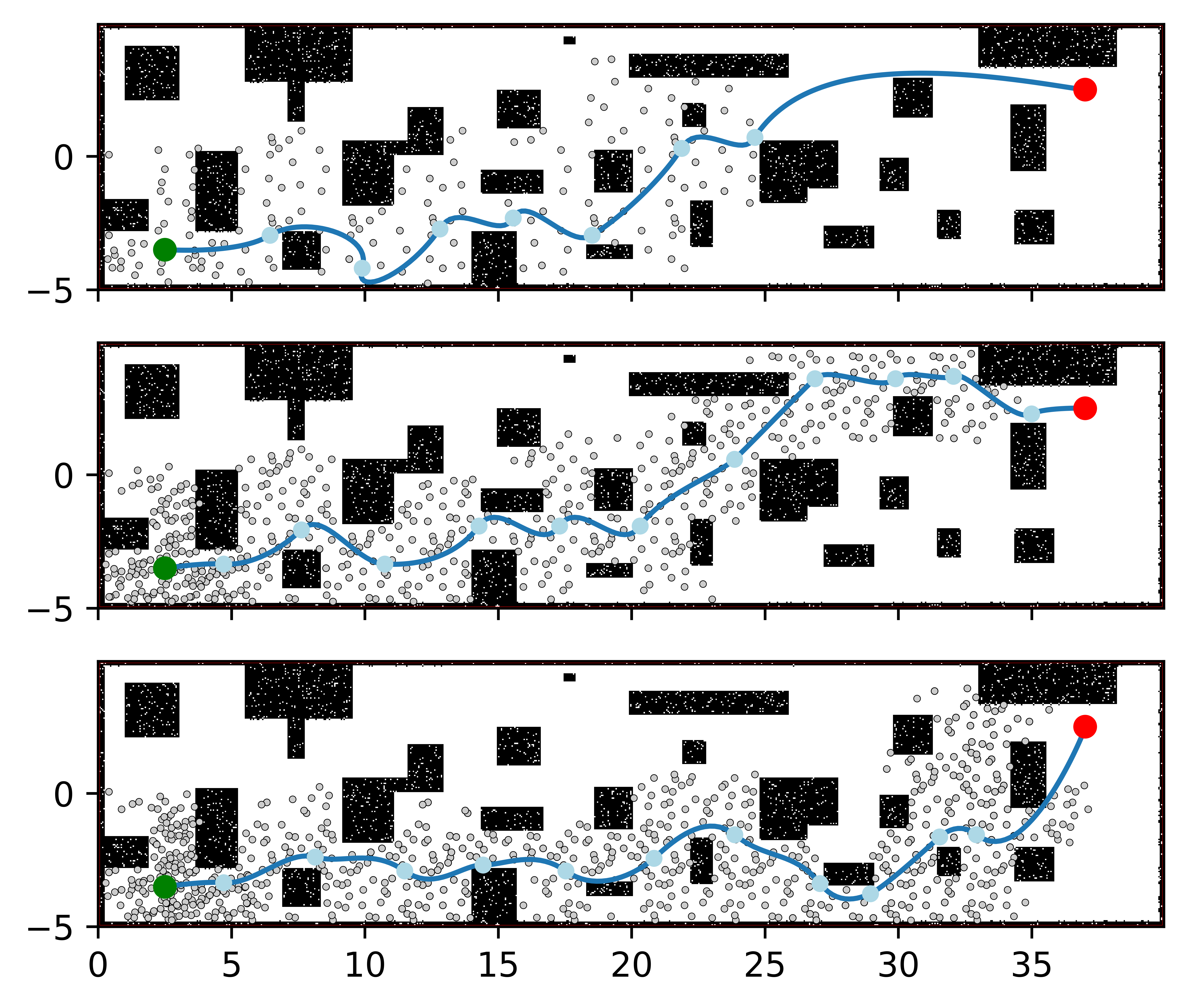}
	\caption{This figure shows the graph search planning results for the quadrotor system over three motion primitive graphs constructed with our method, but with variable dispersion. The green circles represent vertices in the graph, and the grey dots are other states added to the open list of the A* algorithm. We can see that the lowest dispersion plan is able to more optimally go through a narrow corridor that the other two miss, at the expense of more computation. (top) Dispersion = 134, path cost = 2345, number of collision checks = 599 (middle) Dispersion = 112, path cost = 2300, number of collision checks = 1851 (bottom) Dispersion = 104, path cost = 2281, number of collision checks = 3954.}
	\label{fig:planning}
	\vspace{-5mm}
\end{figure}

\section{Evaluation}

\subsection{Systems Considered}
To evaluate the planner described in Section \ref{sec:method}, we primarily consider two systems for which we have access to the steering function, the Reeds-Shepp Car and the quadrotor double integrator model. The quadrotor system is a non-symmetric system with drift.

\subsubsection{Reeds-Shepp Car}
The Reeds-Shepp car \cite{Reeds1990} is a simple kinematic car model with bounded turning. It has a three-dimensional state $(x,y,\theta)$. $J(\vx_i,\vx_j)$ for the RS-car is defined to be the path length. It is a \emph{symmetric} system, $J(\vx_i,\vx_j) = J(\vx_j,\vx_i)$, so our more expansive definition of dispersion is incidental. It is useful for illustration purposes since it is lower dimensional and symmetric. Its analytical optimal steering function is described in \cite{LaValle2006}.

\subsubsection{Quadrotor Double Integrator} \label{sec:quadrotor}
We use the same dynamical system and cost function as \cite{Liu2017}, in order to make a direct comparison to it as a baseline for search-based planning with motion primitives. 
The system and its cost functional are defined as follows, with $\rho$ as a user set constant:
\begin{align}
\dot{\vx} &= A\vx+B\vu\nonumber,\\
A &= \begin{bmatrix}
    \mathbf{0} & \mathbf{I}_2 \\
    \mathbf{0} & \mathbf{0}
    \end{bmatrix}, \quad B = \begin{bmatrix} \mathbf{0} \\ \mathbf{I}_2\end{bmatrix}
\end{align}
\begin{align}
    L(\vx(t), \vu(t), t) =  \|u(t)\|^2 + \rho t 
\end{align}

Since Algorithm 1 requires that we do not fix the duration of the motion primitive, the optimal steering function is computed by solving a bi-level trajectory optimization problem. The inner optimization problem is a standard QP trajectory optimization, and the outer problem is a one dimensional line search over the trajectory duration.
\subsection{Planning Results}
To evaluate the utility and efficacy of using Algorithm \ref{alg:disp} to construct graphs for motion planning, we examine the effect of changing dispersion, and compare dispersion and number of graph search collision checks (as a non-hardware specific comparison of computation time) as compared to baseline motion primitives as described in \cite{Liu2017}, with uniform sampling in the input space over regular time segments.

\subsubsection{Effect of increasing dispersion}
As we can see in Figure \ref{fig:planning}, reducing dispersion in $\setV$ leads to a more complete algorithm, with the top graph having only barely enough samples to find a plan in this map, while the bottom sampling is much more dense. Additionally, the trajectories become more optimal as we reduce the dispersion of $\setV$. Our method provides a single parameter to tune to increase both completeness and optimality of the planner, at the expense of more graph search computation, which causes slower planning time. In a practical context, we could generate graphs for many values of dispersion offline, and iteratively reduce online dispersion until we find a trajectory.

\subsubsection{Comparison to search-based planning with uniform-input sampling motion primitives}
Figure \ref{fig:heatmap2} shows a comparison of cost from a point in the state space to the set $\setV$ for our method as compared to the baseline method for a similar number of vertices and similar overall dispersion. As we can see, more states in our version are nearer to a state in $\setV$. The advantage in this is connected to our completeness proof in section \ref{sec:proof}; more of our vertices are nearby to states that may be on the optimal path.

Figure \ref{fig:planning2} shows a direct comparison between planning with uniform input sampling versus planning with our method, with equivalently optimal output paths. We can see the in this case that the uniform sampling expands many more states (leading to more collision checks and computational burden), though it arrives at an equally optimal outcome. One explanation for these results is that uniform input sampling oversamples high input states, since the cost function includes the norm of the input, leading to many nodes which are never expanded. 

 Table \ref{tab:comparisons} presents the averaged results of planning with both methods on twenty randomized but similar in difficulty maps (not pictured). They both use the same simple heuristic and termination conditions. This shows how our method provides a single parameter to change if a plan is not found to increase completeness and optimality, dispersion. By contrast, uniform input sampling requires setting both the time duration of the motion primitive, and the branching factor of the graph, which are non-intuitive to tune.  Many combinations of branching factor and time duration do not find paths in any or all of the maps, as shown by the rows with `N/A'. This includes when the algorithm maxed out and terminated at 100,000 collision checks. By contrast, our method finds a path in every map, with fewer collision checks, but the most optimal plan on the table belongs to the uniform input sampling. However, this could likely be achieved by generating an even lower dispersion graph.

\begin{figure}
	\centering
	\includegraphics[width=0.9\columnwidth]{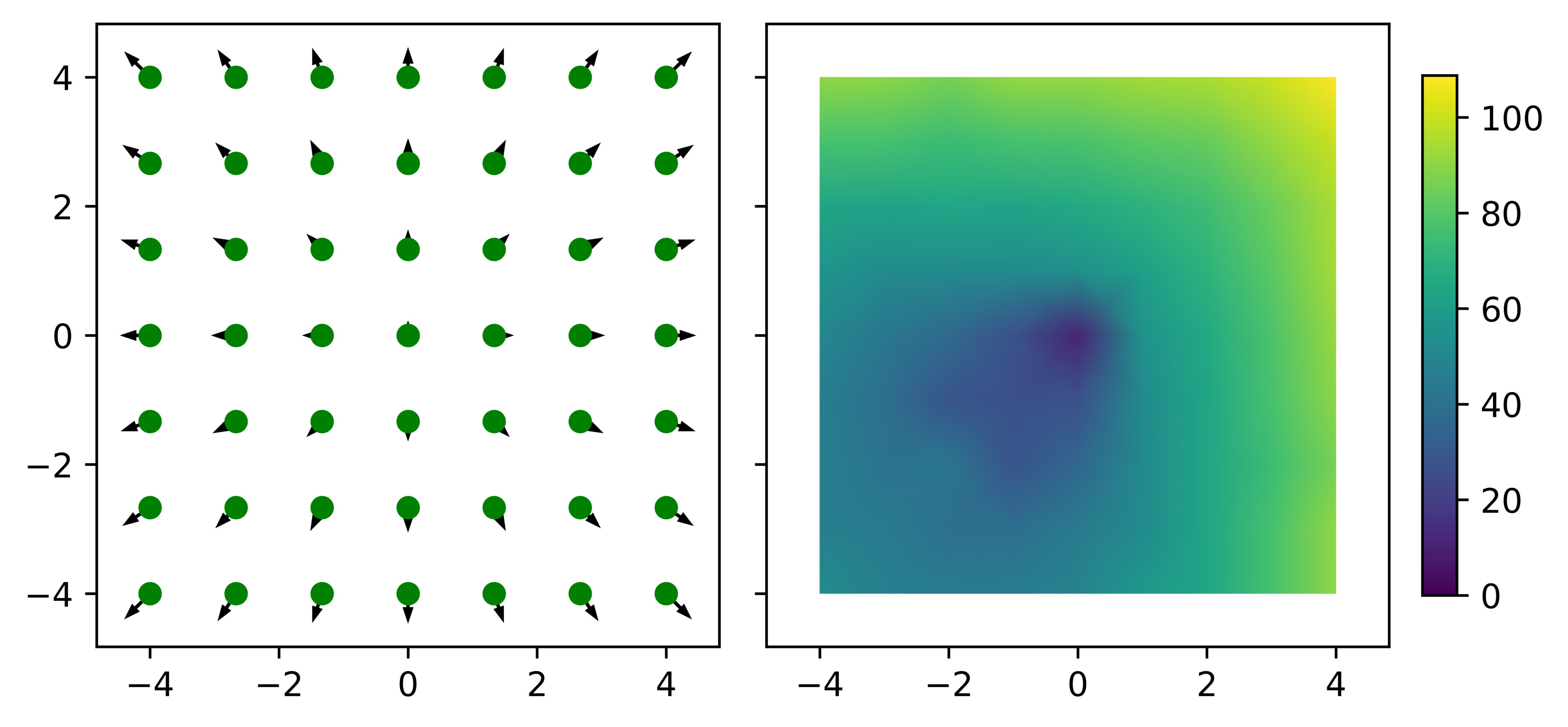}
	\includegraphics[width=0.9\columnwidth]{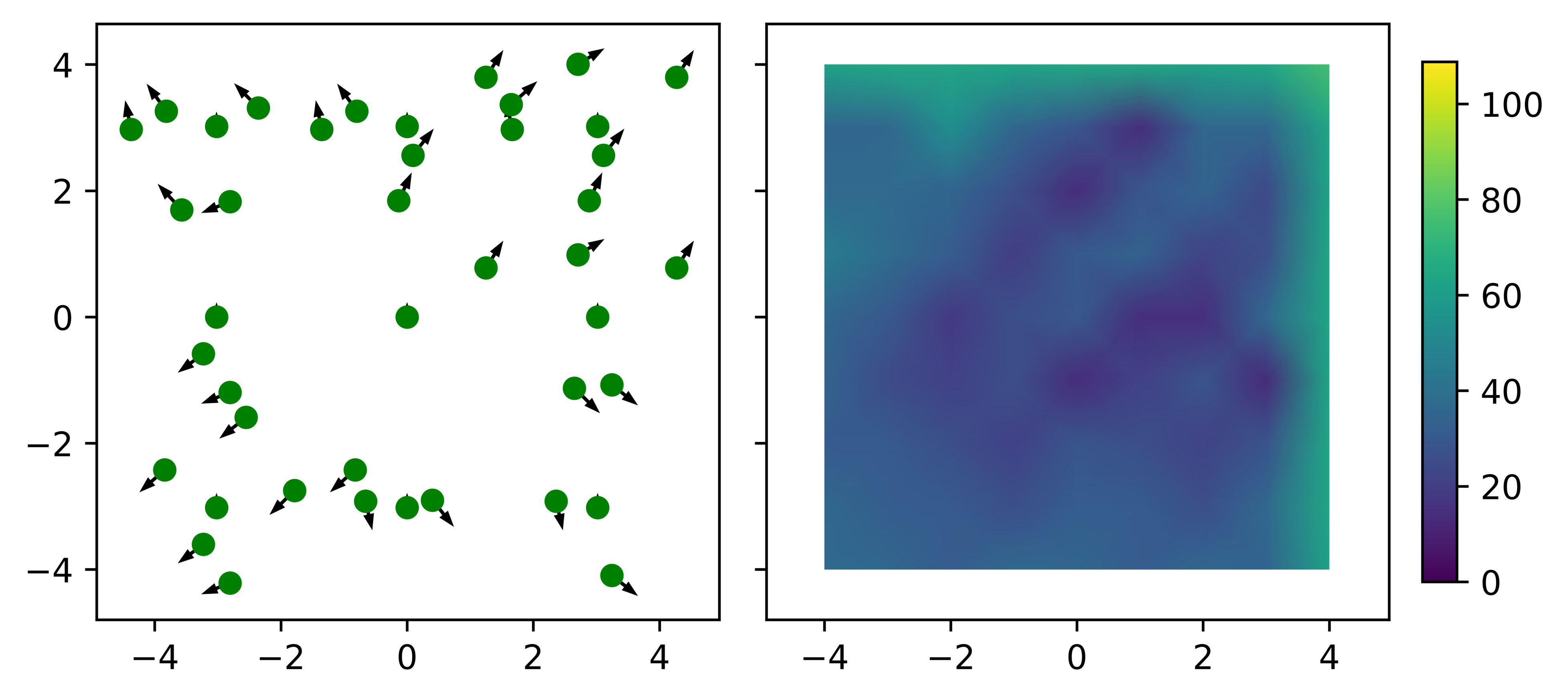}

	\caption{Vertices projected into the configuration space for a quadrotor system motion primitive graph computed with uniform input sampling over a constant time horizon as in \cite{Liu2017}, with their velocities represented by the arrows. The graph is represented at a depth of 1 from a seed vertex at the zero state. (top left)  Vertices output from Algorithm \ref{alg:disp} at a depth of 1 from a seed vertex at the zero state. (bottom left) Minimum cost trajectories from $\setV$ to a slice of $\setXdense$ with coordinates $(x,y,1,1)$ are used to compute a cost surface. $\setV$ is defined by the vertices of the plots on the left. (top right and bottom right).}
	\label{fig:heatmap2}
	\vspace{-5mm}

\end{figure}



\newcolumntype{L}{>{\raggedright\arraybackslash}X}
\begin{table}

  \begin{center} 
    \caption{Planning Comparison on 20 random maps with similar narrow corridor sizes}
   \label{tab:comparisons}
    \begin{tabularx}{\columnwidth}{|L|L||L||L|}

    \hline
    \multicolumn{4}{|c|}{\textbf{Minimum Dispersion Primitives}} \\ \hline
     \multicolumn{2}{|c||}{\textbf{Dispersion}} & \textbf{Avg. \# Collision Checks} & \textbf{Avg. Cost}\\ \hline
       \multicolumn{2}{|c||}{84}  & 3913 & 2897 \\\hline
      \multicolumn{2}{|c||}{95}  & 3361 & 2872 \\\hline
      \multicolumn{2}{|c||}{115}  & 2302 & 3250 \\ \hline
      \multicolumn{2}{|c||}{128} & 1612 & 3332 \\ \hline
      \multicolumn{2}{|c||}{145} & 683 & 3732 \\ \hline
      \multicolumn{4}{|c|}{\textbf{Uniform Input Sampling Primitives}} \\ \hline
      \textbf{Time Duration (s)}& \textbf{Branching Factor (per dimension)} &\textbf{\# Collision Checks} & \textbf{Cost}\\ \hline
      .1 & 3 & $>$100000 & N/A \\ \hline
      .1 & 4 & 8729 & 2525 \\ \hline
      .1 & 5 & $>$100000 &  N/A\\ \hline
      .3 & 3 & 30492 & 2752 \\ \hline
      .3 & 4 & $>$100000 & N/A \\ \hline
      .3 & 5 & $>$100000 & N/A\\ \hline
      .5 & 3 & 3780 &  3183 \\ \hline
      .5 & 4 &  $>$100000 & N/A \\ \hline
      .5 & 5 &  $>$100000 & N/A \\ \hline

    \end{tabularx}

  \end{center}
  	\vspace{-5mm}

\end{table}

\section{Conclusion}
In this paper we have presented a practical application of dispersion to motion planning with dynamically feasible graphs. This method applies to the wide class of systems for which we can solve boundary value problems in free space, and enables us to plan long, dynamically feasible trajectories in cluttered environments.

\FloatBarrier

\bibliographystyle{ieeetr}
\bibliography{motion_primitives}
\end{document}